\documentclass[final,12pt]{colt2025arxiv} 

\title[Blessing of Dimensionality for Approximating Sobolev Classes on Manifolds]{Blessing of Dimensionality for Approximating Sobolev Classes on Manifolds}
\usepackage{times}

\coltauthor{%
 \Name{Hong Ye Tan} \Email{hyt35@cam.ac.uk}\\
 \addr Department of Applied Mathematics and Theoretical Physics\\
 University of Cambridge, UK
 \AND
 \Name{Subhadip Mukherjee} \Email{smukherjee@ece.iitkgp.ac.in}\\
 \addr Department of Electronics and Electrical Communication Engineering\\ IIT Kharagpur, India
 \AND
 \Name{Junqi Tang} \Email{j.tang.2@bham.ac.uk}\\
 \addr University of Birmingham
 \AND
 \Name{Carola-Bibiane Sch\"onlieb} \Email{cbs31@cam.ac.uk}\\
 \addr Department of Applied Mathematics and Theoretical Physics\\
 University of Cambridge, UK%
}

\usepackage{amsmath,amsfonts,bm}

\def\eqref#1{equation~\ref{#1}}

\def\1{\bm{1}}

\DeclareMathAlphabet{\mathsfit}{\encodingdefault}{\sfdefault}{m}{sl}
\SetMathAlphabet{\mathsfit}{bold}{\encodingdefault}{\sfdefault}{bx}{n}

\def\gA{{\mathcal{A}}}

\def\gC{{\mathcal{C}}}

\def\gF{{\mathcal{F}}}

\def\gH{{\mathcal{H}}}

\def\gM{{\mathcal{M}}}

\def\gR{{\mathcal{R}}}
\def\gS{{\mathcal{S}}}

\def\gX{{\mathcal{X}}}

\newcommand{\R}{\mathbb{R}}

\DeclareMathOperator{\dist}{dist}

\DeclareMathOperator*{\argmin}{arg\,min}

\DeclareMathOperator{\sign}{sign}

\DeclareMathOperator{\diam}{diam}

\newcommand{\Ric}{\mathrm{Ric}}
\newcommand{\vol}{\mathrm{vol}}

\DeclareMathOperator{\inj}{inj}
\DeclareMathOperator{\supp}{supp}
\newcommand*\diff{\mathop{}\!\mathrm{d}}

\usepackage{hyperref}
\usepackage{url}
\usepackage{mathtools}
\usepackage{cleveref}
\usepackage{xcolor}
\usepackage{tikz}
\usepackage{soul}
\newtheorem{claim}{Claim}
\newtheorem*{claim*}{Claim}

\begin{document}

\maketitle

\begin{abstract}
The manifold hypothesis says that natural high-dimensional data lie on or around a low-dimensional manifold. The recent success of statistical and learning-based methods in very high dimensions empirically supports this hypothesis, suggesting that typical worst-case analysis does not provide practical guarantees. A natural step for analysis is thus to assume the manifold hypothesis and derive bounds that are independent of any ambient dimensions that the data may be embedded in. Theoretical implications in this direction have recently been explored in terms of generalization of ReLU networks and convergence of Langevin methods. In this work, we consider optimal uniform approximations with functions of finite statistical complexity. While upper bounds on uniform approximation exist in the literature using ReLU neural networks, we consider the opposite: lower bounds to quantify the fundamental difficulty of approximation on manifolds. In particular, we demonstrate that the statistical complexity required to approximate a class of bounded Sobolev functions on a compact manifold is bounded from below, and moreover that this bound is dependent only on the intrinsic properties of the manifold, such as curvature, volume, and injectivity radius. 


\end{abstract}
\section{Introduction}
Data is ever growing, especially in the current era of machine learning. However, dimensionality is not always beneficial, and having too many features can confound simpler underlying truths. This is sometimes referred to as the curse of dimensionality \citep{altman2018curse}. A classical example is manifold learning, which is known to scale exponentially in the intrinsic dimension \citep{narayanan2009sample}. In the current paradigm of increasing dimensionality, standard statistical tools and machine learning models continue to work, despite the high ambient dimensions arising in cases such as computational imaging \citep{wainwright2019high}. One possible assumption to elucidate this phenomenon comes from the manifold hypothesis, also known as concentration of measure or the blessing of dimensionality \citep{bengio2013representation}. This states that real datasets are actually concentrated on or near low-dimensional manifolds, independently of the ambient dimension that the data is embedded in.



In this work, we explore the consequences of the manifold hypothesis through the lens of approximation theory and statistical complexity. For a class of functions with infinite statistical complexity, we consider a nonlinear width (Definition \ref{def:nonlinearwidth}) in terms of how well it can be approximated in $L^p$ with function classes of finite statistical complexity. We consider how difficult it is to optimally approximate classes of functions with functions of finite statistical complexity in terms of $L^p$ distance. In particular, \Cref{thm:mainResult} demonstrates that on a Riemannian manifold, the optimal error incurred by approximating a bounded Sobolev class using function classes of finite statistical complexity can be lower bounded using only the \emph{implicit} properties of the manifold. 


\subsection{Related Works}
The manifold hypothesis is sometimes replaced with the ``union of manifolds'' hypothesis, where the component manifolds are allowed to have different intrinsic dimension \citep{vidal2011subspace,brown2022verifying}. For estimating the intrinsic dimension, we refer to \citep{pope2021intrinsic,block2021intrinsic,levina2004maximum,fefferman2016testing}; for representing the manifold or dimension reduction, we refer to \citep{lee2007nonlinear,kingma2013auto,connor2021variational,tishby2015deep,shwartz2017opening}.

\textbf{Intrinsic dimension estimation.} Methods for empirically testing the manifold hypothesis typically involve assuming the samples follow some statistical process, where the dimension parameter is then estimated from samples using maximum likelihood estimation of distances between points \citep{pope2021intrinsic,block2021intrinsic,levina2004maximum}. \citet{fefferman2016testing} provides a classical algorithm to test whether a set of points can be described by a manifold with sufficient regularity properties. 

\textbf{Learning the manifold/dimension reduction.} In modern datasets, a reasonable proxy to the image manifold hypothesis is to have a representation of the low-dimensional structure, constructed from finitely many samples. \citet{fefferman2016testing} provides a classical algorithm to test whether a set of points can be described by a manifold with sufficient regularity properties. Classical methods to find nonlinear low-dimensional manifolds include locally linear embedding \citep{roweis2000nonlinear}, Isomap \citep{tenenbaum2000global}, eigenmaps \citep{belkin2001laplacian}, and topological properties \citep{niyogi2008finding,niyogi2011topological}, with more comprehensive reviews given in \citep{lee2007nonlinear}. A more modern approach uses generative networks by using the latent space as the manifold, which bypasses using the intrinsic dimension itself as an algorithmic parameter \citep{wang2016auto,demers1992non,nakada2020adaptive}.

\textbf{Manifold-driven architectures.} Driven by the manifold hypothesis, several machine learning approaches consider enforcing a network output to be low-dimensional. Common examples are variational auto-encoders, which consist of an ``encoder'' network mapping from the input to a latent space, and ``decoder'' network mapping from the latent space to an output \citep{kingma2013auto,connor2021variational}. By restricting the dimensionality of the latent space, the output manifold will automatically be restricted. Other methods include bottleneck layers in ResNets, which relate to the information bottleneck tradeoff between compression and prediction \citep{tishby2015deep,shwartz2017opening}. 

\textbf{Application: Langevin mixing times.} For an isometrically embedded manifold, \citet{block2020fast} bounds a log-Sobolev constant for probability measures supported on the manifold that are absolutely continuous with respect to the volume measure, mollified with Gaussian densities in the ambient space. \citet{wang2020fast} demonstrate linear convergence of the Kullback-Leibler divergence with rates depending only on the intrinsic dimension for the geodesic Langevin algorithm, which incorporates the Riemannian metric into the noise in the unadjusted Langevin algorithm. 

\citet{devore1989optimal} show that for any continuously parameterized set of functions that can uniformly $\varepsilon$-approximate the unit ball of the Sobolev space $W^{k,\infty}(\R^d)$, the dimension of the parameter set must scale as $\Omega(\varepsilon^{-d/k})$. \cite{gao2019convergence} shows that any class of functions that can robustly interpolate $n$ samples has Vapnik-Chervonenkis (VC) dimension at least $\Omega(nD)$. \citet{bolcskei2019optimal} show lower bounds for the sparsity of a deep neural network for approximating function classes in $L^2(\R^d)$. 
%
\citet{chen2019efficient,chen2022nonparametric} provide approximation rates, empirical risk estimates and generalization bounds of ReLU networks for H\"older functions on manifolds, assuming isometric embedding in Euclidean space. \citet{labate2023low} consider generalization of the class of ReLU networks for H\"older functions on the manifold using the Johnson--Lindenstrauss lemma. 

On the unit hypercube, \cite{yang2024nearly} addresses the complexity of approximating a Sobolev function \textit{constructively} with ReLU DNNs by showing an upper bound on the VC dimension and pseudo-dimension of derivatives of neural networks based on the number of layers, input dimension, and maximum width. \citet{park2020minimum,kim2023minimum,hanin2017approximating} consider lower bounds for the minimum width required for ReLU and ReLU-like networks to $\varepsilon$-approximate $L^p$ functions on Euclidean space and the unit hypercube. 

This work derives lower bounds (\Cref{thm:mainResult,thm:mainResult2}) on the \textit{statistical complexity} in terms of the \textit{nonlinear width}, c.f. Definition \ref{def:nonlinearwidth}, required to approximate Sobolev functions on \textit{compact Riemannian manifolds}. Sobolev functions define expressive classes of functions that can model many physical problems, while also having sufficient regularity properties allowing for functional analysis. In other words, we consider the difficulty of modelling physical problems over structured datasets with simple function classes. 


\section{Background}\label{sec:background}
\subsection{Pseudo-Dimension as Complexity}
We consider a concept of statistical complexity called the pseudo-dimension \citep{pollard2012convergence,anthony1999neural}. This extends the classical concept of Vapnik--Chervonenkis (VC) dimension from indicator-valued to real-valued functions.
\begin{definition}\label{def:pseudodim}
    Let $\gH$ be a class of real-valued functions with domain $\gX$. 
    Let $X_n = \{x_1,...,x_n\}\subset \gX$, and consider a collection of real numbers $s_1,...,s_n \subset \R^n$. When evaluated at each $x_i$, a function $h \in \gH$ will lie on one side\footnote{We adopt the notation of $\sign(0) = +1$ for well-definedness, but the other option is equally valid.} of the corresponding $x_i$, i.e. $\sign(h(x_i) - s_i) = \pm 1$. The vector of such sides $(\sign(h(x_i) - s_i))_{i=1}^n$ is thus an element of $\{\pm1\}^n$.
    
    We say that $\gH$ \emph{P-shatters} $X_n$ if there exist real numbers $s_1,...,s_n$ such that all possible sign combinations are obtained, i.e.,
    \begin{equation*}
        \{(\sign(h(x_i) - s_i))_i \mid h \in \gH\} = \{\pm 1\}^n.
    \end{equation*}
    The \emph{pseudo-dimension} $\dim_p(\gH)$ is the cardinality of the largest set that is P-shattered:
    \begin{equation}
        \dim_p(\gH) = \sup \left\{n \in \mathbb{N} \mid \exists \{x_1,..., x_n\}\subset \gX \text{ that is P-shattered by } \gH\right\}.
    \end{equation}
\end{definition}
The VC dimension is defined similarly, but without the biases $s_i$ and with $\gH$ being a class of binary functions taking values in $\{\pm 1\}$. The pseudo-dimension satisfies similar properties as the VC dimension, such as coinciding with the standard notion of dimension for vector spaces of functions.

\begin{figure}[t]
    \centering
    \includegraphics[width=0.8\linewidth,trim={0 23px 0 30px},clip]{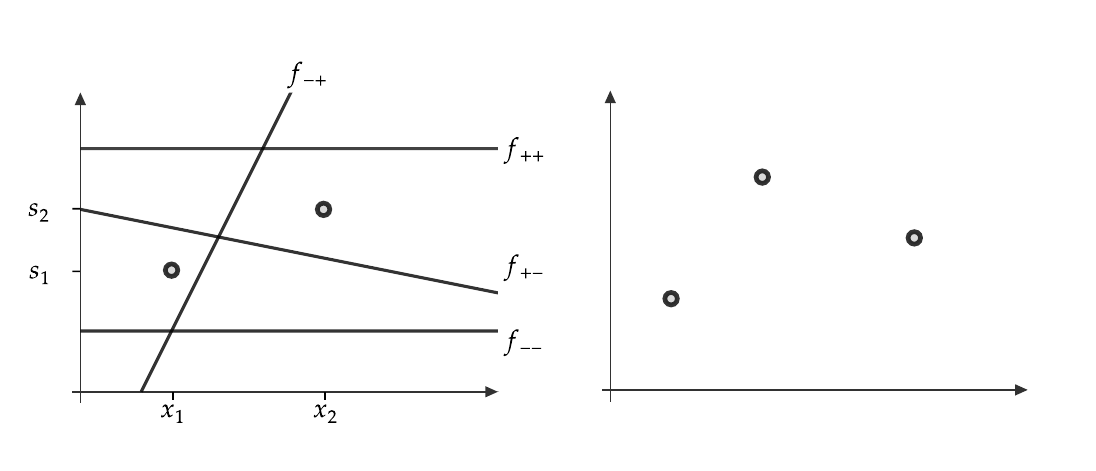}\vspace{-4px}
    \caption{For the class of affine 1D functions $\{x \mapsto ax+b \mid a, b \in \R\}$, this choice of $\{x_1, x_2\} \subset \R$ and $s_1, s_2 \in \R$ on the left is P-shattered by the affine functions $f_{\pm \pm}$. However, there is no arrangement of three points that is P-shattered by affine functions, e.g. the arrangement on the right would not have a function that goes below the left and right points but above the middle point. Therefore, the pseudo-dimension of affine 1D functions is 2.}
    \label{fig:enter-label}
\end{figure}

\begin{proposition}[{\citealt[Thm. 11.4]{anthony1999neural}}]
    If $\gH'$ is an $\R$-vector space of real-valued functions, then $\dim_p(\gH') = \dim(\gH')$ as a vector space. In particular, if $\gH$ is a subset of a vector space $\gH'$ of real-valued functions, then $\dim_p(\gH) \le \dim(\gH')$.
\end{proposition}

Like the VC dimension and other statistical complexity measures such as the Rademacher or Gaussian complexity, low complexity leads to better generalization properties of empirical risk minimizers \citep{bartlett2002rademacher}. One example is as follows, where a precise definition of sample complexity can be found in \Cref{app:sampComp}. 
\begin{proposition}[{\citealt[Thm. 19.2]{anthony1999neural}}]\label{prop:pDimSampComp}
    Let $\gH$ be a class of functions mapping from a domain $X$ into $[0,1] \subset \R$, and that $\gH$ has finite pseudo-dimension. Then the $(\epsilon,\delta)$-sample complexity (Definition \ref{def:samcomp}) is bounded by
    \begin{equation}
        m_L(\epsilon,\delta) \le \frac{128}{\epsilon^2} \left(2\dim_p(\gH) \log \left(\frac{34}{\epsilon}\right) + \log \left(\frac{16}{\delta}\right)\right).
    \end{equation}
\end{proposition}

To compare the approximation of one function class by another, we consider a nonlinear width induced by a normed space. 
\begin{definition}[Nonlinear $n$-width]\label{def:nonlinearwidth}
    Let $\gF$ be a normed space of functions. Given two subsets $F_1, F_2 \subset \gF$, the \emph{(asymmetric) Hausdorff distance} between the two subsets is the largest distance between elements of $F_1$ and their closest element in $F_2$:
    \begin{equation}
        \dist(F_1, F_2; \gF) = \sup_{f_1 \in F_1} \inf_{f_2 \in F_2} \|f_1 - f_2\|_{\gF}.
    \end{equation}
    
    For a subset $F\subset \gF$, the \emph{nonlinear $n$-width} is given by the optimal (asymmetric) Hausdorff distance between $F$ and $\gH^n$, infimized over classes $\gH^n$ in $\gF$ with $\dim_p(\gH^n) \le n$:
    \begin{equation}\label{eq:nWidth}
        \rho_n(F, \gF) \coloneqq \inf_{\gH^n} \dist(F, \gH^n; \gF) = \inf_{\gH^n} \sup_{f \in F} \inf_{h \in \gH^n} \|f-h\|_\gF.
    \end{equation}
\end{definition}
This width measures the complexity in terms of how closely the entire function class can be approximated with another class of finite pseudo-dimension. This is useful in cases where $F$ has infinite pseudo-dimension, and the nonlinear $n$-width acts as a surrogate measure of complexity, given by how well $F$ can be approximated by classes of finite pseudo-dimension. In \Cref{sec:mainResult}, we provide a lower bound on the nonlinear $n$-width of a bounded Sobolev class of functions. In terms of neural network approximation, these lower bounds complement existing approximation results of ReLU networks, which effectively provide an upper bound on the width by using the class of (bounded width, layers, and parameters) ReLU networks as the finite pseudo-dimension approximating class.

\subsection{Riemannian Geometry}
The manifold hypothesis can be readily expressed in terms of Riemannian geometry. A quick review and notation are given in \Cref{app:Riemannian}, and we further refer to \citet{bishop2011geometry,gallot2004riemannian}. Throughout, we will assume that our Riemannian manifold is finite-dimensional, compact, without boundary, and connected. We note that the connectedness assumption can be dropped by working instead on each connected component.

We state the celebrated Bishop--Gromov theorem \citep{petersen2006riemannian,bishop1964relation}. This is an essential volume-comparison theorem, used to tractably bound the volume of balls as they grow.

\begin{theorem}[Bishop--Gromov]\label{thm:BisGro}
    Let $(M,g)$ be a complete $d$-dimensional Riemannian manifold whose Ricci curvature is bounded below by $\Ric \ge (d-1)K$, for some $K \in \R$. Let $M^d_K$ be the complete $d$-dimensional simply connected space of constant sectional curvature $K$, i.e. a sphere, Euclidean space, or scaled hyperbolic space if $K>0$, $K=0$, $K<0$ respectively. Then for any $p\in M$ and $p_K \in M_K^d$, we have that
    \begin{equation}
        \phi(r) = \vol_M(B_r(p)) / \vol_{M^d_K}(B_r(p_K))
    \end{equation}
    is non-increasing on $(0,\infty)$. In particular, $\vol_M(B_r(p)) \le \vol_{M^d_K}(B_r(p_K))$.
\end{theorem}

We note that in a space of constant sectional curvature $M^d_K$, the volume of a ball of radius $r$ does not depend on the center. Without loss of generality, we write $\vol_{M^d_K}(B_r)$ to mean $\vol_{M^d_K}(B_r(p_K))$ for any point $p_K \in M^d_K$. Bishop--Gromov can be specialized as integrals in hyperbolic space. 


\begin{corollary}[{\citealt{block2020fast,ohta2014ricci}}]\label{thm:BisGroSinh}
    Let $(M,g)$ be a complete $d$-dimensional Riemannian manifold such that $\Ric \ge (d-1) K$ for some $K<0$. For any $0<r<R$, we have
    \begin{equation}
        \frac{\vol_M(B_R(x))}{\vol_M(B_r(x))} \le \frac{\int_0^R s^{d-1}}{\int_0^r s^{d-1}},\quad s(u) = \sinh(u\sqrt{|K|}).
    \end{equation}
\end{corollary}
We additionally need the following definitions, bounded in \Cref{prop:coveringNum} and Lemma \ref{lem:PackingL1}. 
\begin{definition}[Packing number]
    For a metric space $(M,d)$ and radius $\varepsilon>0$, the \emph{packing number} $N_\varepsilon(M)$ is the maximum number of points $x_1,...,x_n \in M$ such that the open balls $B_\varepsilon(x_i)$ are disjoint. The \emph{$\varepsilon$-metric entropy} $\mathcal{M}_\varepsilon$ is the maximum number of points $x_1,...,x_m \in M$ such that $d(x_i, x_j) \ge \varepsilon$ for $i \ne j$. Note that $\mathcal{M}_{2\varepsilon} \le N_\varepsilon \le \mathcal{M}_\varepsilon.$
\end{definition}
\subsection{Sobolev Functions on Manifolds}
We now define the bounded Sobolev ball on manifolds, which will be the subject of approximation in the next section. There are different ways to define the Sobolev spaces on manifolds due to the curvature differing only by constants, and we consider the variant presented in \citet{hebey2000nonlinear}. 

\begin{definition}[{\citealt[Sec. 2.2]{hebey2000nonlinear}}]
    Let $(M,g)$ be a smooth Riemannian manifold. For integer $k$, $p \ge 1$, and smooth $u:M \rightarrow \R$, define by $\nabla^k u$ the $k$'th covariant derivative of $u$, and $|\nabla^k u|$ its norm, defined in a local chart as 
    \begin{equation}
        |\nabla^k u|^2 = g^{i_1 j_1}... g^{i_k j_k} (\nabla^k u)_{i_1...i_k} (\nabla^k u)_{j_1...j_k},
    \end{equation}
    using Einstein's summation convention where repeated indices are summed. Define the set of admissible test functions (with respect to the volume measure) as 
    \begin{equation}
    \mathcal{C}^{k,p}(M) \coloneqq \left\{u \in \mathcal{C}^\infty(M) \mid \forall j=0,...,k,\, \int_M |\nabla^j u|^p \diff \vol_M < +\infty \right\},
    \end{equation}
    and for $u \in \mathcal{C}^{k,p}(M)$, the Sobolev $W^{k,p}$ norm as
    \begin{equation}
        \|u\|_{W^{k,p}} \coloneqq \sum_{j=0}^k \left(\int_M |\nabla^j u|^p \diff \vol_M \right)^{1/p} = \sum_{j=0}^k \|\nabla^j u\|_p
    \end{equation}
    The Sobolev space $W^{k,p}(M)$ is defined as the completion of $\mathcal{C}^{k,p}$ under  $\|\cdot\|_{W^{k,p}}$.
\end{definition}

It can be shown that Sobolev functions on a compact Riemannian manifold share similar embedding properties as in Euclidean space, and we refer to \citet{hebey2000nonlinear,aubin2012nonlinear} for a more detailed treatment of such results. 
%
%
We additionally adopt the following definition of a bounded Sobolev ball, providing a compact space of functions to approximate.
\begin{definition}
    For constant $C\ge 0$, the bounded Sobolev ball $W^{k,p}(C; M)$ is given by the set of all functions with covariant derivatives bounded in $L^p = L^p(M, \vol_M)$ by $C$:
    \begin{equation}
        W^{k,p}(C;M) = \left\{u \in W^{k,p}(M) \mid \forall l \le k,\, \|\nabla^l u\|_p \le C \right\}
    \end{equation}
    We write $W^{k,p}(C)$ to mean $W^{k,p}(C;M)$ for ease of notation.
\end{definition}

\section{Main Result}\label{sec:mainResult}
This section begins with a statement of the main approximation result, a lower bound on the nonlinear $n$-width \labelcref{eq:nWidth} of bounded Sobolev balls $W^{1,p}(1)$ in $L^q$. This is followed by a high-level intuition behind the proof, then the proof in detail. The supporting lemmas are deferred to \Cref{app:lems}.
\begin{theorem}\label{thm:mainResult}
    Let $(M,g)$ be a $d$-dimensional compact (separable) Riemannian manifold without boundary. From compactness, there exist real constants $K,\, \inj(M)$ such that:
    \begin{enumerate}
        \item The Ricci curvature satisfies $\Ric \ge (d-1)K$, where $K < 0$;
        \item The injectivity radius is positive, $\inj(M)>0$.
    \end{enumerate}
    For any $1 \le p,q \le +\infty$, the nonlinear width of $W^{1,p}(1)$ satisfies the lower bound for sufficiently large $n$:
    \begin{equation}
        \rho_n(W^{1,p}(1),L^q(M)) \ge C(d,K,\vol(M),p,q) (n + \log n)^{-1/d}.
    \end{equation}
    The (explicit) constant is independent of any ambient dimension that $(M,g)$ may be embedded in.
\end{theorem}
Note that this statement does not refer to any ambient dimension or embeddings, and can be defined on abstract manifolds. This theorem should be contrasted with \citet[Thm. 1]{maiorov1999degree}, which exhibits a similar bound $n^{-1/d}$ for the bounded Sobolev space on the unit hypercube $[0,1]^d$. The additional $\log n$ term is necessary due to the curvature of the space, but can otherwise be absorbed into the constant. We also note that it is possible to perform this analysis in the case of positive curvature and derive better bounds. For higher regularity functions, a similar result is available but without explicit constants. The proof of this extension is deferred to \Cref{app:mainProof2}.

\begin{theorem}\label{thm:mainResult2}
    Assume the setup of \Cref{thm:mainResult}. For any $1 \le p,q \le +\infty$, the nonlinear width of $W^{k,p}(1)$ satisfies the lower bound for sufficiently large $n$:
    \begin{equation}
        \rho_n(W^{k,p}(1),L^q(M)) \ge C(d,g,p,q,k) (n + \log n)^{-k/d}.
    \end{equation}
    where the constant depends on the boundedness of the manifold geometry.
\end{theorem}

There are two major difficulties in converting the proof of \citet{maiorov1999degree} to the manifold setting, both arising from curvature. Firstly, the original proof uses a partition into hypercubes to construct the desired counterexample. As such hypercube partitions do not possess nice properties on manifolds, we instead consider a packing of geodesic balls, which does not fully cover the manifold and loosens the bound. The second major difference is the lack of global information, particularly for geodesic balls of the same radius which can have different volumes at different points, introducing additional constants into the final bound. 

\subsection{Proof Sketch}
\begin{enumerate}
    \item[\textbf{Step 1.}] We consider a class of simple functions, defined as sums of cutoff functions with disjoint supports. The class of simple functions is such that the $L^1$-norm of each component is large within the class of bounded $W^{1,p}(1)$ functions.
    \item[\textbf{Step 2.}] An appropriate subset of the simple functions is then taken, that is isometric to the hypercube graph $\{\pm1\}^m$. Since we can find an $\ell_1$-well separated subset of the hypercube graph, there exists an $L^1$-well separated subset of our simple functions.
    \item[\textbf{Step 3.}] The $L^1$ separation of the constructed set prevents approximation with classes of insufficiently large pseudo-dimension. This step uses an exponential lower bound on the metric entropy and a polynomial upper bound from Bishop--Gromov to derive a contradiction.
    \item[\textbf{Step 4.}] We conclude that the optimal approximation with function classes with bounded pseudo-dimension must incur an error, bounded from below as in the theorem statement. We conclude the proof by combining all the inequalities from the lemmas and Step 3.
\end{enumerate}

\subsection{Proof of Theorem \ref{thm:mainResult}}
In the following, $L^p$ spaces will be on $(M,g)$ with respect to the underlying volume measure. Recall the definition of the bounded class of $W^{1,p}$ functions:
\begin{equation}
    W^{1,p}(C) = \left\{u \in W^{1,p}(M) \mid \|u\|_{L^p},\|\nabla u\|_{L^p} \le C\right\}.
\end{equation}

\textbf{Step 1. Defining the base function class}. Fix a radius $0< r < \inj(M)$, which will be chosen appropriately later. Fix a maximal packing of geodesic $r$-balls, say with centers $p_1,....,p_{N_r}$, where $N_r = N_r^{\text{pack}}(M)$ is the packing number. By definition, $B_r(p_i)$ are disjoint for $i=1,...,N_r$. From \Cref{prop:coveringNum}, the packing number satisfies the following where $D=\diam(M)$:
\begin{equation}\label{eq:packingBounds}
    \frac{\vol(M)}{\vol_{M^d_K}(B_{2r})} \le N_r^{\text{pack}} \le  \frac{\vol_{M^d_K}(B_D)}{\vol_{M^d_K}(B_{r})}.
\end{equation}

For each ball $B_r(p_i)$, using \citep[Cor. 3]{azagra2007smooth}, we can construct a $\mathcal{C}^\infty$ function $\phi'_i:M \rightarrow [0,r/4]$ with support $\supp(\phi'_i) \subset B_r(p_i)$ such that $|\nabla \phi'_i| \le 1$ pointwise and 
\begin{equation}\label{eq:defPhiPrime}
    \phi'_i(p) = \begin{cases}
        r/4, & d(p, p_i) \le r/2;\\
        0, & d(p, p_i) \ge r.
    \end{cases}
\end{equation}
%
%
From \labelcref{eq:defPhiPrime}, we have the $L^1$ lower bound
\begin{equation}
    \|\phi'_i\|_{1} \ge (r/4) \vol_M(B_{r/2}(p_i)).
\end{equation}
Moreover, we have the $L^p$ bounds on $\phi'_i$ and $\nabla \phi'_i$,
\begin{gather}
     \|\phi'_i\|_{p} \le (r/4) \vol_M(B_{r}(p_i))^{1/p},\quad
    \|\nabla \phi'_i\|_{p} \le \vol_M(B_{r}(p_i))^{1/p}.
\end{gather}
Therefore, for $r<4$, we have that $\phi'_i \in W^{1,p}(\vol_M(B_{r}(p_i))^{1/p})$. Define
a non-negative function $\phi_i$ with support in $B_r(p_i)$ satisfying:
\begin{equation}\label{eq:phiL1LowerBdd}
    \phi_i \coloneqq \frac{\phi_i'}{\vol_M(B_{r}(p_i))^{1/p}},\quad \|\phi_i\|_{1} \ge (r/4) \frac{\vol_M(B_{r/2}(p_i))}{\vol_M(B_{r}(p_i))^{1/p}},\quad \phi_i \in W^{1,p}(1).
\end{equation}
Moreover, $\phi_i = r/(4 \vol_M(B_r(p_i))^{1/p})$ on $B_{r/2}(p_i)$. We now consider the function class
\begin{equation}
    F_r = \left\{ f_a = \frac{1}{N_r^{1/p}}\sum_{i=1}^{N_r} a_i \phi_i \,\middle\vert\, a_i \in \{\pm 1\},\, i=1,...,N_r\right\}.
\end{equation}
Since the sum is over functions of disjoint support, we have that $\|f_a\|_p, \|\nabla f_a\|_p \le 1$, and thus each element of $F_r$ also lies in $W^{1,p}(1)$. Moreover, every element $f_a \in F_r$ satisfies the $L^1$ lower bound using \labelcref{eq:phiL1LowerBdd}:
\begin{equation}
    \|f_a\|_{1} \ge \frac{r}{4N_r^{1/p}} \sum_{i=1}^{N_r} \frac{\vol_M(B_{r/2}(p_i))}{\vol_M(B_{r}(p_i))^{1/p}},\quad \forall f_a \in F_r.
\end{equation}

\noindent\textbf{Step 2. $L^1$-well-separation of $F_r$.}
Consider the following lemma, which shows the existence of a large well-separated subset of $\ell_1^m$.
\begin{lemma}[{\citealt[Lem. 2.2]{lorentz1996constructive}}]\label{lem:ell1Separation}
    There exists a set $G \subset \{\pm 1\}^m$ of cardinality at least $2^{m/16}$ such that for any $v \ne v' \in G$, the distance $\|v-v'\|_{\ell_1^m} \ge m/2$. In particular, any two elements differ in at least $m/4$ entries.
\end{lemma}

In particular, let $G \subset \{\pm 1\}^{N_r}$ be well separated by the above lemma. Denote by $F_r(G) = \{f_a \in F_r \mid a \in G\}.$ the subset of $F_r$ corresponding to these indices. For the specific choice of separated $G \subset \{\pm 1\}^{N_r}$ in the above lemma, we claim the following well-separation of $F_r(G)$, proved in \Cref{app:claimProof}.
\begin{claim}\label{claim:FrSep}
    There exists a constant $C_1(r)>0$ such that for any $f \ne f' \in F_r(G)$, we have
    \begin{equation}\label{eq:mainSepC1}
        \|f-f'\|_{1} \ge C_1(r) > 0.
    \end{equation}
    Moreover, the following constant works:
    \begin{equation}\label{eq:C1Def}
        C_1(r) = \frac{r N_r^{1-1/p}\int_0^{r/2}s^{d-1}}{8 \int_0^{r}s^{d-1} } \inf_{i \in [N_r]}\left[\vol_M(B_{r}(p_i))^{1-1/p}\right].
    \end{equation}
\end{claim}
This shows $L^1$-well separation of the subset $F_r(G) \subset F_r \subset W^{1,p}(1)$, which consist of sums of disjoint cutoff functions. The key is to contrast this with the metric entropy bounds in Lemma \ref{lem:PackingL1}, by showing that $F_r(G)$ is difficult to approximate with function classes of low pseudo-dimension.

\noindent\textbf{Step 3a. Construction of well-separated bounded set.} Let $\gH^n$ be a given set of $\vol_M$-measurable functions with $\dim_p(\gH^n) \le n$. Let $\varepsilon>0$. Denote
\begin{equation}\label{eq:defDelta}
    \delta = \sup_{f \in F_r(G)} \inf_{h \in \gH^n} \|f-h\|_{1} + \varepsilon = \dist(F_r(G), \gH^n, L^1(M)) + \varepsilon.
\end{equation}

Define a projection operator $P : F_r(G) \rightarrow \gH^n$, mapping $f \in F_r(G)$ to any element $Pf$ in $\gH^n$ such that $\|f - Pf\|_{1} \le \delta$. We introduce a (measurable) clamping operator $\mathcal{C}$ for a function $f$:
\begin{gather}
    \beta_i = r/(4 \vol_M(B_r(p_i))^{1/p} N_r^{1/p}),\quad i=1,...,N_r,\\
    (\gC f)(x) = \begin{cases}
        -\beta_i, & x \in B_r(p_i) \text{ and } f(x) < -\beta_i;\\
        f(x), & x \in B_r(p_i) \text{ and } -\beta_i \le f(x) \le \beta_i;\\
        \beta_i, & x \in B_r(p_i) \text{ and } f(x) > \beta_i;\\
        0, & \text{otherwise}. \\
    \end{cases}
\end{gather}
Note that $\beta_i$ are the bounds of $f_a \in F_r$ in the balls $B_r(p_i)$. Now consider the set of functions $\gS \coloneqq \gC P F_r(G)$. Suppose $f \ne f' \in F_r(G)$. We show separation in $\gS$ using triangle inequality:
\begin{align}\label{eq:CPfCPf}
    \|\gC P f - \gC P f'\|_{1} \ge \|f-f'\|_{1} - \|f - \gC P f\|_{1} - \|f' - \gC P f'\|_1.
\end{align}
For any $a \in G$, we have that $f_a \le \beta_i$ in $B_r(p_i)$, and both $f_a$ and $\gC P f_a$ are zero on $M \setminus \bigsqcup_i B_r(p_i)$. We thus have that for any $x \in M$ and any $f_a \in F_r(G)$, 
    $|f_a(x) - \gC P f_a(x)| \le |f_a(x) - P f_a(x)|$. 
This inequality holds for $x \in B_r(p_i)$ since $\mathcal{C}$ clamps $P f_a(x)$ towards $[-\beta_i,\beta_i]$, and holds trivially on $M \setminus \bigsqcup_i B_r(p_i)$. Integrating and by definition of $P$, we have that for any $f_a \in F_r(G)$,
\begin{equation}\label{eq:fCPf}
    \|f_a - \gC P f_a\|_{1} \le \|f_a - P f_a\|_{1} \le \delta.
\end{equation}
Using \labelcref{eq:CPfCPf,eq:fCPf,eq:mainSepC1}, we thus have separation
\begin{align}\label{eq:SeparatedS}
    \|\gC P f - \gC P f'\|_{1} \ge \|f-f'\|_{1} -2\delta
    \ge C_1(r) - 2\delta. 
\end{align}

\noindent\textbf{Step 3b. Minimum distance by contradiction.} Suppose for contradiction that $\delta \le C_1(r)/4$. Then from \labelcref{eq:SeparatedS}, we have
\begin{equation}
    \|\gC P f - \gC P f'\|_{1} \ge C_1(r)/2.
\end{equation}
The separation implies that the $\gC P f$ are distinct for distinct $f \in F_r(G)$, thus $|\mathcal{S}| = |G| \ge 2^{N_r/16}$. 

Define $\alpha = C_1(r)/2$. Consider the metric entropy in $L^1$, as given in Lemma \ref{lem:PackingL1}. By construction \labelcref{eq:SeparatedS}, $\mathcal{S}$ itself is an $\alpha$-separated subset in $L^1$ as any two elements are $L^1$-separated by $\alpha$, so
\begin{equation}\label{eq:EntropyLowerBd}
    \gM_\alpha(\mathcal{S}, L^1(\vol_M)) \ge 2^{N_r/16}.
\end{equation}

We now wish to obtain an upper bound on $\gM_\alpha(\gS, L^1)$ using Lemma \ref{lem:PackingL1}. From the definition of pseudo-dimension, we have $\dim_p(\gC PF_r(G)) \le \dim_p(PF_r(G))$, since any P-shattering set for $\gC PF_r(G)$ will certainly P-shatter $PF_r(G)$. Since $PF_r(G) \subset \gH^n$, we have $\dim_p(PF_r(G)) \le \dim_p(\gH^n) \le n$. Thus $\dim_p(\mathcal{S}) = \dim_p(\gC PF_r(G)) \le n$. $\gS$ is $L^1$-separated with distance at least $\alpha$, and moreover consists of elements that are bounded by $\beta \coloneqq \sup_i \beta_i$. Lemma \ref{lem:PackingL1} now gives:
\begin{equation}\label{eq:EntropyUpperBd}
    M_\alpha(\mathcal{S}, L^1(\vol_M)) \le e(n+1) \left(\frac{4e\beta\vol(M)}{\alpha}\right)^n.
\end{equation}
Intuitively, $N_r \sim r^{-d}$, so the lower bound \labelcref{eq:EntropyLowerBd} is exponential in $r$. Meanwhile, $\alpha$ and $\beta$ are both polynomial in $r$, so the upper bound \labelcref{eq:EntropyUpperBd} is polynomial in $r$. So for sufficiently small $r$, we have a contradiction with the supposition that $\delta \le C_1(r)/4$. We now show this formally. Recall:
\begin{equation}
    \beta = \sup_{i \in [N_r]} \frac{r}{4 \vol_M(B_r(p_i))^{1/p} N_r^{1/p}},\quad \alpha = \frac{r N_r^{1-1/p}\int_0^{r/2}s^{d-1}}{16 \int_0^{r}s^{d-1} } \inf_{i \in [N_r]}\left[\vol_M(B_{r}(p_i))^{1-1/p}\right].
\end{equation} 

Note that the supremum in $\beta$ and the infimum in $\alpha$ is attained by the same $i \in [N_r]$, namely, the $p_i$ that has smallest $\vol_M(B_r(p_i))$. Combining \labelcref{eq:EntropyLowerBd} and \labelcref{eq:EntropyUpperBd}, where $s(u) = \sinh(u \sqrt{|K|})$,
\begin{align}
    2^{N_r/16} &\le e(n+1) \left(\frac{4 e\beta \vol(M)}{\alpha}\right)^n \notag\\
    & = e(n+1) \left(\frac{4e\vol(M) \sup_{i \in [N_r]} [r/(4 \vol_M(B_r(p_i))^{1/p} N_r^{1/p})] }{\frac{r N_r^{1-1/p}\int_0^{r/2}s^{d-1}}{16 \int_0^{r}s^{d-1} } \inf_{i \in [N_r]}\left[\vol_M(B_{r}(p_i))^{1-1/p}\right]}\right)^n \notag\\
    &= e(n+1) \left(16 e \frac{\vol(M)\int_0^{r} s^{d-1} }{N_r\int_0^{r/2} s^{d-1}}\sup_{i,j \in [N_r]} \frac{\vol_M(B_{r}(p_j))^{1/p-1}}{\vol_M(B_r(p_i))^{1/p}}\right)^n \notag \\
    &= e(n+1) \left(16 e \frac{\vol(M)\int_0^{r} s^{d-1} }{N_r\int_0^{r/2} s^{d-1}}\sup_{i \in [N_r]} \left[\vol_M(B_{r}(p_i))^{-1}\right] \right)^n \label{eq:EntropyCombinedBd}
\end{align}
where the equalities come from definition of $\beta$ and $\alpha$ and rearranging, and the last equality from noting the supremum is attained when $i=j \in [N_r]$ minimizes $\vol_M(B_r(p_i))$. The following result lower-bounds the volume of small balls to control the supremum term.
\begin{proposition}[{\citealt[Prop. 14]{croke1980some}}]\label{prop:smallBall}
    For $r \le \inj(M)/2$, the volume of the ball $B_{r}(p)$ satisfies
    \begin{equation}
        \vol_M(B_{r}(p)) \ge C_2(d) r^d,\quad C_2(d) = \frac{2^{d-1} \vol_{M_1^{d-1}}(B_1)^d}{d^d \vol_{M_1^{d}}(B_1)^{d-1}}.
    \end{equation}
\end{proposition}
The volume of the $d$-dimensional hyperbolic sphere with sectional curvature $K$ can be written in terms of the volume of the $d$-dimensional sphere\footnote{The volume of the $d$-dimensional sphere is $2\pi^{d/2}/\Gamma(d/2)$, where $\Gamma$ is Euler's gamma function.}: 
\begin{equation}
    \vol_{M_{K}^d}(B_\rho) = \vol_{M_1^{d}}(B_1) \int_{t=0}^\rho \left(\frac{\sinh(\sqrt{|K|}t)}{\sqrt{|K|}}\right)^{d-1} \diff t.
\end{equation}
Note that $x \le \sinh(x) \le 2x$ for $x \in [0,2]$. Therefore, for $\rho \le 2/\sqrt{|K|}$, we have $\sinh(\sqrt{|K|}t) \le 2\sqrt{|K|}t$. We thus have that
\begin{equation}\label{eq:MdKVolupperbd}
\begin{split}
    \vol_{M_{K}^d}(B_\rho) &= \vol_{M_1^{d}}(B_1) \int_{t=0}^\rho \left(\frac{\sinh(\sqrt{|K|}t)}{\sqrt{|K|}}\right)^{d-1} \diff t \\ &< \vol_{M_1^{d}}(B_1)2^{d-1}\rho^d/d = C_3(d) \rho^d,\quad C_3(d) \coloneqq \vol_{M_1^d}(B_1) 2^{d-1}/d.
\end{split}
\end{equation}
Moreover,
\begin{equation}\label{eq:MdKUpperBd}
    \frac{\int_0^{r} s^{d-1} }{\int_0^{r/2} s^{d-1}} = \frac{\int_0^{r} \sinh(\sqrt{|K|}u)^{d-1} \diff u}{\int_0^{r/2} \sinh(\sqrt{|K|}u)^{d-1}\diff u} \le 2^d \quad \text{ for $r < 1/\sqrt{K}$.}
\end{equation}
We continue the inequality \labelcref{eq:EntropyCombinedBd} for $r < 1/\sqrt{|K|}$:
\begin{align}
    2^{N_r/16} &\le e(n+1) \left(16 e \frac{\textcolor{blue}{\vol(M)}\int_0^{r} s^{d-1} }{\textcolor{blue}{N_r}\int_0^{r/2} s^{d-1}}\textcolor{red}{\sup_{i \in [N_r]} \left[\vol_M(B_{r}(p_i))^{-1}\right]} \right)^n &&  \\
    & \le e(n+1) \left(16 e \textcolor{blue}{\vol_{M^d_K}(B_{2r})} \frac{\int_0^{r} s^{d-1} }{\int_0^{r/2} s^{d-1}} \textcolor{red}{C_2^{-1}r^{-d}} \right)^n && \text{using \textcolor{blue}{\labelcref{eq:packingBounds}}, \textcolor{red}{Prop. \labelcref{prop:smallBall}}} \\
    & \le e(n+1) \left(16 e C_3 (2r)^d \frac{\int_0^{r} s^{d-1} }{\int_0^{r/2} s^{d-1}} C_2^{-1}r^{-d} \right)^n && \text{using \labelcref{eq:MdKVolupperbd}} \\
    & \le e(n+1) \left(2^{d+4} e C_3  \frac{\int_0^{r} s^{d-1} }{\int_0^{r/2} s^{d-1}} C_2^{-1} \right)^n \\
    & \le e(n+1) \left(2^{2d+4} e C_3 C_2^{-1} \right)^n =  e(n+1) C_4(d)^n, && \text{using \labelcref{eq:MdKUpperBd}}
\end{align}
where $ C_4 = C_4(d) \coloneqq 2^{2d+4} e C_3 C_2^{-1}$.
We get a contradiction if 
\begin{equation}\label{eq:NrContradictionStep}
    N_r > 16 \left[n \log_2 C_4 + \log_2\left(e(n+1)\right)\right].
\end{equation}
Recalling the lower bound \labelcref{eq:packingBounds} on $N_r$ and using \labelcref{eq:MdKUpperBd},
\begin{align}\label{eq:NrExplicitLowerBound}
    N_r &\ge \frac{\vol(M)}{\vol_{M^d_K}(B_{2r})} > \frac{\vol(M)}{C_3 (2r)^d}.
\end{align}
Take the following choice of $r$:
\begin{equation}\label{eq:ChoiceR}
    r = \min\left\{\frac{1}{2}\left(16\frac{C_3}{\vol(M)} \left[n \log_2 C_4 + \log_2(e(n+1))\right]\right)^{-1/d},\, \frac{1}{\sqrt{|K|}},\, \frac{\inj(M)}{2},\, 4 \right\}.
\end{equation}
Using \labelcref{eq:NrExplicitLowerBound}, this choice of $r$ satisfies the contradiction condition \labelcref{eq:NrContradictionStep}. Note $r \sim (n + \log n)^{-1/d}$. The constants $C_3, C_4$ depend only on $d$.

\textbf{Step 4. Concluding contradiction.} This choice of $r$ contradicts the assumption that $\delta \le C_1(r)/4$. Therefore, we must have that $\delta > C_1/4$. Since the choice of $r$ is independent of the choice of $\varepsilon>0$ taken at the start of Step 3a, we have that 
\begin{equation}\label{eq:FrGLowerBdC1}
    \dist(F_r(G), \gH^n, L^1(\vol_M)) \ge C_1(r)/4,
\end{equation}
where $r$ is chosen as in \labelcref{eq:ChoiceR}. We obtain the chain of inequalities
\begin{align}
    \dist(W^{1,p}(1), \gH^n, L^q) &\ge \dist(W^{1,p}(1), \gH^n, L^1) \vol(M)^{1/q-1} \notag\\
    &\ge \dist(F_r(G), \gH^n, L^1)\vol(M)^{1/q-1} \notag\\
    &\ge C_1(r) \vol(M)^{1/q-1}/4 \notag\\
    &= \frac{r N_r^{1-1/p}\int_0^{r/2}s^{d-1}}{32 \int_0^{r}s^{d-1} } \inf_{i \in [N_r]}\left[\vol_M(B_{r}(p_i))^{1-1/p}\right] \vol(M)^{1/q-1}, \label{eq:distbound}
\end{align}
where the first inequality comes from H\"older's inequality $\|u\|_1 \le \|u\|_q \vol(M)^{1-1/q}$, the second inequality from $F_r(G) \subset W^{1,p}(1)$, the third from \labelcref{eq:FrGLowerBdC1} and the equality from definition \labelcref{eq:C1Def} of $C_1(r)$. We conclude with recalling the bounds \labelcref{eq:MdKUpperBd}, \labelcref{eq:NrExplicitLowerBound}, and \Cref{prop:smallBall}. We have
\begin{align*}
    \dist(W^{1,p}(1), \gH^n, L^q) &\ge \frac{r}{32} \underbrace{\left(\frac{\vol(M)}{C_3 (2r)^d}\right)^{1-1/p}}_{\text{\labelcref{eq:NrExplicitLowerBound}}} \underbrace{2^{-d}}_{\text{\labelcref{eq:MdKUpperBd}}} \underbrace{\left[C_2 r^d\right]^{1-1/p}}_{\text{Prop. \ref{prop:smallBall}}} \vol(M)^{1/q-1} \\
    &= C_5(d, \vol(M),p,q) r,
\end{align*}
where $C_5 = 2^{-2d-5+d/p} \vol(M)^{1/q-1/p} (C_2 C_3^{-1})^{1-1/p}.$ Moreover, the constant $C_5$ and choice of $r$ are independent of $\gH^n$. Taking infimum over all choice of $\gH^n$ with $\dim_p(\gH^n) \le n$ and using \labelcref{eq:ChoiceR}, we have
\begin{equation}
    \rho_n(W^{1,p}(1), L^q) \ge C_5(d, \vol(M),p,q) r \gtrsim (n + \log n)^{-1/d}.
\end{equation}%
\subsection{Relationship with existing bounds}
\cite{yarotsky2017error} presents various results for approximating Sobolev balls on the unit hypercube. We note that the notion of complexity here is the number of neurons/weight.
\begin{proposition}[{\citealt[Sec 3.2]{yarotsky2017error}}]
    Consider the Sobolev space $W^{k,\infty}([0,1]^d)$. Given $\varepsilon>0$, there is a ReLU network architecture of complexity $\mathcal{O}(\varepsilon^{-d/k} \log(1/\varepsilon))$ that is able to approximate any function in the unit ball of $W^{k,\infty}([0,1]^d)$ with error $\varepsilon$.
\end{proposition}
We note that the corresponding result of \Cref{thm:mainResult2} in the case of $M=[0,1]^d$ in \citet{maiorov1999degree} gives the complexity lower bound $\Omega(\varepsilon^{-d/k})$ required to approximate the unit ball $W^{k,\infty}(1;[0,1]^d)$ with error $\varepsilon$. In particular, this states that ReLU neural networks with this architecture are asymptotically nearly optimal in achieving the best possible approximation.

In the context of manifolds, the following result gives the best available bounds in terms of ambient dimension dependence. The result considers using ReLU neural networks to approximate H\"older functions, which are closely related to Sobolev functions.
\begin{proposition}[{\citealt[Thm 1]{chen2022nonparametric}}]
    For a $d$-dimensional compact Riemannian manifold $M$ without boundary isometrically embedded in $\R^D$, the unit ball of the H\"older space $C^{s,\alpha}(M)$ can be $\varepsilon$-approximated with a ReLU network architecture with complexity $\mathcal{O}(\varepsilon^{-\frac{d}{s+\alpha}} \log (\frac{1}{\varepsilon}) + D \log \frac{1}{\varepsilon} + D \log D)$. The constant depends on $d, s, \vol(M)$, the reach and diameter of the embedding, and the bounded geometry of $M$.
\end{proposition}
We note that by the manifold version of Morrey's inequality, the Sobolev space of $W^{k,\infty}$ compactly and continuously embeds into $C^{k-1,1-\delta}$ for any $\delta \in (0,1)$ \citep{aubin2012nonlinear}. This result thus gives a complexity upper bound of $\mathcal{O}(\varepsilon^{-\frac{d}{k-\delta}})$, nearly matching the lower bound of $\Omega(\varepsilon^{-d/k})$  in \Cref{thm:mainResult2} up to log factors. However, our lower bound is independent of the ambient dimension $D$.

\section{Conclusion}
This work provides a theoretical motivation to further explore the manifold hypothesis. We show that the problem of approximating a bounded class of Sobolev functions depends only on the intrinsic properties of the supporting manifold. More precisely, the approximation error of the bounded $W^{k,p}$ space with respect to bounded pseudo-dimension classes is shown to be at least $(n+\log n)^{-k/d}$, where $d$ is the intrinsic dimension of the underlying manifold. Since generalization error is linear in pseudo-dimension, this provides an ambient-dimension-free lower-bound on generalization error. This is in contrast to many works in the literature that provide constructive upper bounds on generalization error based on ReLU approximation properties that still depend on the embedding of the manifold in ambient Euclidean space. Followup work could consider Rademacher or Gaussian complexity, or alternative geometries such as Finsler manifolds \citep{busemann2005geometry}.



\bibliography{refs}
\crefalias{section}{appendix} 
\appendix
\newpage
\section{Sample complexity}\label{app:sampComp}
For completeness, we briefly formalize the sample complexity bound \Cref{prop:pDimSampComp}, based on \citep{anthony1999neural}.

\begin{definition}[{\citealt[Def. 16.4]{anthony1999neural}}]
    For a set of functions $F$, an \emph{approximate sample error minimizing (approximate-SEM)} algorithm $\gA(z, \epsilon)$ takes any finite number of samples $z = (x_i, y_i)_{i=1}^m$ in $\cup_{m=1}^\infty (X \times \R)^m$ and an error bound $\epsilon>0$, and outputs an element $f \in F$ satisfying
    \begin{equation}
        \gR(\gA(F,z)) < \inf_{f \in F} \gR(f) + \epsilon , \quad \gR_m(f) = \frac{1}{m}\sum_{i=1}^m (f(x_i)-y_i)^2.
    \end{equation}
\end{definition}

\begin{definition}[{\citealt[Def. 16.1]{anthony1999neural}}]\label{def:samcomp}
    For a set of functions $F$ mapping from domain $X$ to $[0,1]$, a \emph{learning algorithm} $L$ for $F$ is a function taking any finite number of samples, 
    \begin{equation}
        L:\bigcup_{m=1}^\infty (X \times \R)^m \rightarrow F
    \end{equation}
    with the following property. For any $\epsilon,\delta \in (0,1)$, there is an integer (\emph{sample complexity}) $m_0(\epsilon,\delta)$ such that if $m \ge m_0(\epsilon,\delta)$, the following holds for any probability distribution $P$ on $X \times [0,1]$.

    If $z$ is a training sample of length $m$ according to the product distribution $P^m$ (i.i.d. samples), then with probability at least $1-\delta$, the function $L(z)$ output by $L$ is such that
    \begin{equation}
        \mathbb{E}_{(x,y)\sim P}[L(z)(x)-y]^2< \inf_{f \in F} \mathbb{E}_{(x,y)\sim P}[f(x)-y]^2 + \epsilon.
    \end{equation}
    In other words, given $m\ge m_0$ training samples, the squared-risk of the learning algorithm's output is $\epsilon$-optimal with probability at least $1-\delta$.
\end{definition}

Observe that the approximate-SEM algorithm works on the empirical risk, while the learning algorithm works on the risk. Relating the two thus gives generalization bounds. The formal version of \Cref{prop:pDimSampComp}, based on \citep{anthony1999neural} is now given as follows.

\begin{proposition}[{\citealt[Thm. 19.2]{anthony1999neural}}]
    Let $\gH$ be a class of functions mapping from a domain $X$ into $[0,1] \subset \R$, and that $\gH$ has finite pseudo-dimension. Let $\mathcal{A}$ be any approximate-SEM algorithm, and define for samples $z$, $L(z) = \mathcal{A}(z, 16/\sqrt{\text{length}(z)})$. Then $L$ is a learning algorithm for $\gH$, and its sample complexity is bounded as follows:
    \begin{equation}
        m_L(\epsilon,\delta) \le \frac{128}{\epsilon^2} \left(2\dim_p(\gH) \log \left(\frac{34}{\epsilon}\right) + \log \left(\frac{16}{\delta}\right)\right).
    \end{equation}
\end{proposition}

\subsection{Relation to our bounds}
In computing a risk minimizer over the Sobolev ball, we need to make two practical simplifications: namely parameterizing the Sobolev ball (into a function class of finite pseudo-dimension), and in simplifying the risk from (typically) an expectation into an empirical version. Our bound targets the former approximation, while the aforementioned sample complexity bounds targets the latter generalization problem. To formalize this, we have the relationship between the three quantities:
\begin{equation*}
    \argmin_{f \in W^{1,p}(1)} \mathcal{R}(f) \longleftrightarrow  \argmin_{f \in \mathcal{H}_n} \mathcal{R}(f) \longleftrightarrow \argmin_{f \in \mathcal{H}_n} \mathcal{R}_m(f),
\end{equation*}
where $\mathcal{H}_n$ is some function class with pseudo-dimension at most $n$. For the sake of exposition, we make some additional assumptions and work in the worst-case.

Suppose that the (expected) risk $\mathcal{R} : W^{1,p}(1) \rightarrow \R$ is Lipschitz continuous in the space of functions, such as $\mathcal{R} = \mathbb{E}_\mu[\|f(x)-y\|^2]$ for some sufficiently regular probability measure $\mu \in \mathcal{P}(\mathcal{M} \times \R)$ and measurement space $\mathcal{Y}$. Consider the set of minimizers $\mathcal{G} \coloneqq \argmin_{f \in W^{1,p}(1)} \mathcal{R}(f)$, and assume that $\argmin_{f \in W^{1,p}(1)} \mathcal{R}(f) = \argmin_{f \in L^q} \mathcal{R}(f)$, i.e., risk minimizers in $L^q$ are also in $W^{1,p}$.

Take $\mathcal{H}_n \subset L^q$ to be an optimal approximating class of pseudo-dimension at most $n$. In the worst case, we have that the $L^q$ distance between $\gH_n$ and minimizers $g \in \mathcal{G}$ is bounded from below by some $\epsilon = \epsilon(n)>0$ as given by \Cref{thm:mainResult}. Assuming further some strong convexity conditions, this gives that the minimizer $f^* \in \argmin_{f \in \mathcal{H}_n} \mathcal{R}(f)$ has risk at least $c\epsilon$ for some strong-convexity constant $c$. Adding this worst-case risk with the worst-case risk of the $(\epsilon,\delta)$-sample complexity bounds, we have that an empirical risk minimizer may have even greater risk. 

We make these assumptions for the sake of exposition in the worst-case; note however that \Cref{thm:mainResult} considers the furthest element in $W^{1,p}(1)$ from $\gH_n$, and that minimizers in $\gH_n$ and $W^{1,p}(1)$ may be closer together. 

\section{Riemannian Geometry}\label{app:Riemannian}
\begin{definition}
    A $d$-dimensional \emph{Riemannian manifold} is a real smooth manifold $M$ equipped with a Riemannian metric $g$, which defines an inner product on the tangent plane $T_p M$ at each point $p \in M$. We assume $g$ is smooth, i.e. for any smooth chart $(U,x)$ on $M$, the components $g^{ij} = g(\frac{\partial}{\partial x_i},\frac{\partial}{\partial x_j}):U \rightarrow \R$ are $\mathcal{C}^\infty$.

    A manifold is without boundary if every point has a neighborhood homeomorphic to an open subset of $\R^d$. For a point $p \in M$, let $B_r(p)$ be the metric ball around $p$ in $M$ with radius $r>0$.

    The \emph{sectional (or Riemannian) curvature} takes at each point $p \in M$, a tangent plane $P \subset T_p M$ and outputs a scalar value. The \emph{Ricci curvature} (function) $\Ric(v) \equiv \Ric(v,v)$ of a unit vector $v \in T_p M$ is the mean sectional curvature over planes containing $v$ in $T_p M$. In particular, for a manifold with constant sectional curvature $K$, we have $\Ric \equiv (d-1) K$. We write $\Ric \ge K$ for $K \in \R$ to mean that $\Ric(v) \ge K$ holds for all unit vectors in the tangent bundle $v \in T M$.

    The \emph{injectivity radius} $\inj(p)$ at a point $p \in M$ is the supremum over radii $r>0$ such that the exponential map defines a global diffeomorphism (nonsingular derivative) from $B_r(0; T_p M)$ onto its image in $M$. The injectivity radius $\inj(M)$ of a manifold is the infimum of such injectivity radii over all points in $M$.

    A Riemannian manifold has a (unique) natural volume form, denoted $\vol_M$. In local coordinates, the volume form is 
    \begin{equation}
        \vol_M = \sqrt{|g|} \diff x_1 \wedge ... \wedge \diff x_d,
    \end{equation}
    where $g$ is the Riemannian metric, and $\diff x_1,..., \diff x_d$ is a (positively-oriented) cotangent basis. We drop the subscripts when taking the volume of the whole manifold $\vol(M) = \vol_M(M)$.
\end{definition}

Intuitively, the sectional curvature controls the behavior of geodesics that are close, and the Ricci curvature controls volumes of small balls. For manifolds of positive sectional curvature such as on a sphere, geodesics tend to converge, and small balls have less volume than Euclidean balls. In manifolds with negative sectional curvature such as hyperbolic space, geodesics tend to diverge, and small balls have more volume than Euclidean balls.

Within the ball of injectivity, geodesics are length-minimizing curves. The injectivity radius defines the largest ball on which the geodesic normal coordinates may be used, where it locally behaves as $\R^d$. This is an intrinsic quantity of the manifold, which does not depend on the embedding.

The volume form can be thought of as a higher-dimensional surface area, where the scaling term $\sqrt{|g|}$ arises from curvature and choice of coordinates. For example, for the 2-sphere $\mathbb{S}^2$ embedded in $\R^3$, the volume form is simply the surface measure, which can be expressed in terms of polar coordinates. 

From the compactness assumption, we have that the sectional (and hence Ricci) curvature is uniformly bounded from above and below \citep[Sec. 9.3]{bishop2011geometry}, and the injectivity radius $\inj(M)$ is positive \citep{cheeger1982finite,grantinjectivity}.

\begin{proposition}
    Let $(M,g)$ be a compact Riemannian manifold without boundary. The following statements hold.
    \begin{enumerate}
        \item (Bounded curvature) The sectional curvature (and hence the Ricci curvature) is uniformly bounded from above and below \citep[Sec. 9.3]{bishop2011geometry}.
        \item (Positive injectivity radius) Let the sectional curvature $K_m$ be bounded by some $|K_M| \le K$. Suppose there exists a point $p \in M$ and constant $v_0 > 0$ such that $\vol_M(B_1(p)) \ge v_0$. Then there exists a positive constant $i_1 = i_1(K,v_0, d)$ such that \citep{cheeger1982finite,grantinjectivity}:
        \begin{equation*}
            \inj(p) \ge i_1 > 0
        \end{equation*}
        In particular, since $M$ is compact, using a finite covering argument, $\inj(M)$ is bounded below by some positive constant.
    \end{enumerate}
\end{proposition}

\section{Packing Lemmas}\label{app:lems}
\begin{proposition}[Packing number estimates]\label{prop:coveringNum}
    Suppose $(M,g)$ has curvature lower-bounded by $K \in \R$, diameter $D$ and dimension $d$. Let $M^d_K$ be the $d$-dimensional model space of constant sectional curvature $K$ (i.e. sphere, Euclidean space, or hyperbolic space). The packing number $N_\varepsilon(M)$ satisfies, where $p_K$ is any point in $M^d_K$:
    \begin{equation}
        \frac{\vol (M)}{\vol_{M^d_K}(B_{2\varepsilon})}\le N_\varepsilon \le \frac{\vol_{M^d_K}(B_D)}{\vol_{M^d_K}(B_{\varepsilon})}
    \end{equation}
\end{proposition}
\begin{proof}
Let $\{p_1,..., p_{N_\varepsilon}\}$ be an $\varepsilon$-packing of $M$. 

    \noindent \textbf{Lower bound.} By maximality, balls of radius $2\varepsilon$ at the $p_i$ cover $M$, so we have by summing over volumes and using Bishop--Gromov:
    \begin{equation}
        \vol(M) \le \sum_{i=1}^{N_{\varepsilon}} \vol_M(B_{2\varepsilon}(p_i)) \le N_\varepsilon \vol_{M^d_K}(B_{2\varepsilon}).
    \end{equation}
    
    \noindent \textbf{Upper bound.} Apply \Cref{thm:BisGro} with $\varepsilon \le D$. We have $\vol_M(B_\varepsilon(p)) \le \vol_{M^d_K}(B_\varepsilon(p_K))$ and $\vol_M(B_D(p_i)) = \vol(M)$ for all $i$. Since the $\varepsilon$-balls are disjoint, we have by finite additivity and Bishop--Gromov:
    \begin{equation}
        \vol_M(M) \ge \sum_{i=1}^{N_\varepsilon} \vol_M(B_{\varepsilon}(p_i)) \ge N_{\varepsilon} \vol(M) \frac{\vol_{M^d_K}(B_{\varepsilon})}{\vol_{M^d_K}(B_D)}.
    \end{equation}
\end{proof}
We additionally consider a bound on the metric entropy for bounded functions.

\begin{lemma}[{\citealt[Cor. 2 and 3]{haussler1995sphere}}]\label{lem:PackingL1}
    For any set $X$, any probability distribution $P$ on $X$, any distribution $Q$ on $\R$, any set $\gF$ of $P$-measurable real-valued functions on $X$ with $\dim_p(\gF) = n < \infty$ and any $\varepsilon > 0$, the $\varepsilon$-metric entropy $\gM_\varepsilon$ (largest cardinality of a $\varepsilon$-separated subset, where distance between any two elements is $\ge \varepsilon$) satisfies:
    \begin{equation}
        \gM_\varepsilon(\gF, \sigma_{P,Q}) \le e(n+1) \left(\frac{2e}{\varepsilon}\right)^n.
    \end{equation}
    Specifically, taking $L^1$ distance, if $\mathcal{F}$ is $P$-measurable taking values in the interval $[0,1]$, we have 
    \begin{equation}
        \gM_\varepsilon(\gF, L^1(P)) \le e(n+1)\left(\frac{2e}{\varepsilon}\right)^n.
    \end{equation}
    If $\sigma$ is instead a finite measure, and $\mathcal{F}$ is $\sigma$-measurable taking values in the interval $[-\beta,\beta]$, then
    \begin{equation}\label{eq:metricEntropyUpperBd}
        \gM_\varepsilon(\gF, L^1(\sigma)) \le e(n+1)\left(\frac{4e\beta\sigma(X)}{\varepsilon}\right)^n.
    \end{equation}
\end{lemma}
\begin{remark}
    The final inequality comes from the second-to-last inequality, by noting that a $\varepsilon$-separated set in $\sigma$ corresponds to a $\varepsilon/\sigma(X)$-separated set in the normalized measure $\sigma/\sigma(X)$, as well as scaling everything by $2\beta$. 
\end{remark}

\section{Separation Claim}\label{app:claimProof}
Here we show \Cref{claim:FrSep} . Recall that $G \subset \{\pm1\}^{N_r}$ is defined to be well-separated by \Cref{lem:ell1Separation}.
\begin{claim*}
    For any $f \ne f' \in F_r(G)$, we have
    \begin{equation}
        \|f-f'\|_{1} \ge \frac{r N_r^{1-1/p}\int_0^{r/2}s^{d-1}}{8 \int_0^{r}s^{d-1} } \inf_{i \in [N_r]}\left[\vol_M(B_{r}(p_i))^{1-1/p}\right].
    \end{equation}
\end{claim*}
\begin{proof}
    Suppose $f \ne f' \in F_r(G)$. In particular, they are generated by multi-indices $a \ne a' \in G$. Consider the set of indices $\mathcal{I} \subset [N_r]$ such that $a_i \ne a'_i$. By construction in \Cref{lem:ell1Separation}, $|\mathcal{I}| \ge N_r/4$. Then the difference between $f$ and $f'$ on $B_r(p_i)$ is $2\phi_i/N_r^{1/p}$ if $i \in \mathcal{I}$, and 0 otherwise. By disjointness of the $B_r(p_i)$, we have
    \begin{align}
        \|f-f'\|_{1} &= \sum_{i \in \mathcal{I}} \frac{2}{N_r^{1/p}}\|\phi_i\|_{1} \ge \frac{r}{2N_r^{1/p}} \sum_{i \in \mathcal{I}}\frac{\vol_M(B_{r/2}(p_i))}{\vol_M(B_{r}(p_i))^{1/p}}\\
        & \ge \sum_{i \in \mathcal{I}} \frac{r\int_0^{r/2}s^{d-1} \vol_M(B_{r}(p_i))}{2N_r^{1/p} \int_0^{r}s^{d-1} \vol_M(B_{r}(p_i))^{1/p}} \qquad \text{where } s(u) = \sinh(u \sqrt{|K|}) \\
        & \ge \frac{r N_r^{1-1/p}\int_0^{r/2}s^{d-1}}{8 \int_0^{r}s^{d-1} } \inf_{i \in \mathcal{I}}\left[\vol_M(B_{r}(p_i))^{1-1/p}\right] \\
        & \ge \frac{r N_r^{1-1/p}\int_0^{r/2}s^{d-1}}{8 \int_0^{r}s^{d-1} } \inf_{i \in [N_r]}\left[\vol_M(B_{r}(p_i))^{1-1/p}\right]\label{eq:FrGL1LowerBd}
    \end{align}
    by the $L^1$-bound on $\phi_i$, Bishop--Gromov (\Cref{thm:BisGroSinh}), and using $|I| \ge N_r/4$ for the inequalities respectively. 
\end{proof}

\section{Proof of Theorem \ref{thm:mainResult2}}\label{app:mainProof2}
\begin{theorem}
    Let $(M,g)$ be a $d$-dimensional compact (separable) Riemannian manifold without boundary. For any $1 \le p,q \le +\infty$, the nonlinear width of $W^{k,p}(1)$ satisfies the lower bound for sufficiently large $n$:
    \begin{equation}
        \rho_n(W^{k,p}(1),L^q(M)) \ge C(d,g,p,q,k) (n + \log n)^{-k/d}.
    \end{equation}
    where the constant depends on the boundedness of the manifold geometry.
\end{theorem}
\begin{proof}
    The main difference to the proof of \Cref{thm:mainResult} is in Step 1, the construction of the base function class. We label minute changes to the proof of \Cref{thm:mainResult} in \textcolor{red}{red}. Recall the definition of the bounded Sobolev class
\begin{equation}
    W^{k,p}(C) = \left\{u \in W^{k,p}(M) \mid \|u\|_{L^p},\|\nabla^i u\|_{L^p} \le C,\, 1 \le i \le k\right\}.
\end{equation}
Instead of using the infimal convolution based construction of \cite{azagra2007smooth} to get an explicit bound on the ratio between the basis function, we can use compactness of the manifold to show boundedness for an arbitrary bump function (without explicit constants). 

\textbf{Step 1. Defining the base function class}. Fix a radius $0< r < \inj(M)$, which will be chosen appropriately later. Fix a maximal packing of geodesic $r$-balls, say with centers $p_1,....,p_{N_r}$, where $N_r = N_r^{\text{pack}}(M)$ is the packing number. 

Let $\chi:\R \rightarrow [0,1]$ be a smooth bump function with support in $[-1,1]$ and satisfying $\chi(r) = 1$ for $r \in [-1/2,1/2]$. For each ball $B_r(p_i)$, construct a $\mathcal{C}^\infty$ function as follows:
\begin{equation}
    \phi_i'(p) = r^k \chi(\|\log_{p_i}(p)\|/r), \quad p = p_i,
\end{equation}
where $\log_{p_i}$ is the logarithmic map at $p_i$. By the compactness of $(M,g)$, the manifold has \emph{bounded geometry}, that is, all terms and derivatives of the curvature tensor are uniformly bounded. Using a uniform bound over $B_r(p_i)$, the covariant derivatives of $\phi_i'$ are all bounded as 
\begin{equation}
    \|\nabla^j \phi_i'\|_p \lesssim_{d,g} r^{k-j} \vol_M(B_r(p_i))^{1/p},\quad j=0,...,k.
\end{equation}
Moreover, 
\begin{equation}
    \|\nabla^j \phi_i'\|_1 \gtrsim_{d,g} r \vol_M(B_{r/2}(p_i)).
\end{equation}


Therefore, for $r<4$, and WLOG dividing by a sufficiently large constant $C = C(d,g)$ depending on the bounded geometry, we have that $\phi'_i \in W^{\textcolor{red}{k},p}(\vol_M(B_{r}(p_i))^{1/p})$. Defining
\begin{equation}
    \phi_i \coloneqq \frac{\phi_i'}{\textcolor{red}{C}\vol_M(B_{r}(p_i))^{1/p}},
\end{equation}
we get a non-negative function $\phi_i$ with support in $B_r(p_i)$ satisfying:
\begin{equation}\label{eq:phiL1LowerBdd2}
    \|\phi_i\|_{1} \ge \textcolor{red}{r^k} \frac{\vol_M(B_{r/2}(p_i))}{\textcolor{red}{C}\vol_M(B_{r}(p_i))^{1/p}},\quad \phi_i \in W^{1,p}(1).
\end{equation}
Moreover, $\phi_i = \textcolor{red}{r^k}/(\textcolor{red}{C}\vol_M(B_r(p_i))^{1/p})$ on $B_{r/2}(p_i)$. We now consider the function class
\begin{equation}
    F_r = \left\{ f_a = \frac{1}{N_r^{1/p}}\sum_{i=1}^{N_r} a_i \phi_i \,\middle\vert\, a_i \in \{\pm 1\},\, i=1,...,N_r\right\}.
\end{equation}
Since the sum is over functions of disjoint support, we have that $\|f_a\|_p, \|\nabla f_a\|_p \le 1$, and thus each element of $F_r$ also lies in $W^{\textcolor{red}{k},p}(1)$. Moreover, every element $f_a \in F_r$ satisfies the $L^1$ lower bound using \labelcref{eq:phiL1LowerBdd2}:
\begin{equation}
    \|f_a\|_{1} \ge \frac{\textcolor{red}{r^k}}{\textcolor{red}{C}N_r^{1/p}} \sum_{i=1}^{N_r} \frac{\vol_M(B_{r/2}(p_i))}{\vol_M(B_{r}(p_i))^{1/p}},\quad \forall f_a \in F_r.
\end{equation}
\noindent\textbf{Step 2. $L^1$-well-separation of $F_r$.}
\Cref{claim:FrSep} can be replaced with the following constant with a very similar proof, using the $L^1$ lower bound of the $\phi_i$.
\begin{claim}
    There exists a constant $C_1(r)>0$ such that for any $f \ne f' \in F_r(G)$, we have
    \begin{equation}
        \|f-f'\|_{1} \ge C_1(r) > 0.
    \end{equation}
    Moreover, the following constant works:
    \begin{equation}\label{eq:C1Def2}
        C_1(r) = \frac{\textcolor{red}{r^k} N_r^{1-1/p}\int_0^{r/2}s^{d-1}}{\textcolor{red}{2C} \int_0^{r}s^{d-1} } \inf_{i \in [N_r]}\left[\vol_M(B_{r}(p_i))^{1-1/p}\right].
    \end{equation}
\end{claim}

\noindent\textbf{Step 3a. Construction of well-separated bounded set.} Let $\gH^n$ be a given set of $\vol_M$-measurable functions with $\dim_p(\gH^n) \le n$. Let $\varepsilon>0$. Denote
\begin{equation}
    \delta = \sup_{f \in F_r(G)} \inf_{h \in \gH^n} \|f-h\|_{1} + \varepsilon = \dist(F_r(G), \gH^n, L^1(M)) + \varepsilon.
\end{equation}

Define a projection operator $P : F_r(G) \rightarrow \gH^n$, mapping any $f \in F_r(G)$ to any element $Pf$ in $\gH^n$ such that
\begin{equation}
    \|f - Pf\|_{1} \le \delta.
\end{equation}
We introduce a (measurable) clamping operator $\mathcal{C}$ for a function $f$:
\begin{gather}
    \beta_i = \textcolor{red}{r^k}/(\textcolor{red}{C} \vol_M(B_r(p_i))^{1/p} N_r^{1/p}),\quad i=1,...,N_r,\\
    (\gC f)(x) = \begin{cases}
        -\beta_i, & x \in B_r(p_i) \text{ and } f(x) < -\beta_i;\\
        f(x), & x \in B_r(p_i) \text{ and } -\beta_i \le f(x) \le \beta_i;\\
        \beta_i, & x \in B_r(p_i) \text{ and } f(x) > \beta_i;\\
        0, & \text{otherwise}. \\
    \end{cases}
\end{gather}
Note that $\beta_i$ are the bounds of $f_a \in F_r$ in the balls $B_r(p_i)$. As before, we have separation
\begin{align}\label{eq:SeparatedS2}
    \|\gC P f - \gC P f'\|_{1} \ge \|f-f'\|_{1} -2\delta
    \ge C_1(r) - 2\delta. 
\end{align}

\noindent\textbf{Step 3b. Minimum distance by contradiction.} Suppose for contradiction that $\delta \le C_1(r)/4$. Then from \labelcref{eq:SeparatedS2}, we have
\begin{equation}
    \|\gC P f - \gC P f'\|_{1} \ge C_1(r)/2.
\end{equation}
In particular, the separation implies that the $\gC P f$ are distinct for distinct $f \in F_r(G)$, thus $|\mathcal{S}| = |G| \ge 2^{N_r/16}$. 

Define $\alpha = C_1(r)/2$. Consider the metric entropy in $L^1$, as given in Lemma \ref{lem:PackingL1}. By construction \labelcref{eq:SeparatedS}, $\mathcal{S}$ itself is an $\alpha$-separated subset in $L^1$ as any two elements are $L^1$-separated by $\alpha$, so
\begin{equation}\label{eq:EntropyLowerBd2}
    \gM_\alpha(\mathcal{S}, L^1(\vol_M)) \ge 2^{N_r/16}.
\end{equation}

We now wish to obtain an upper bound on $\gM_\alpha(\gS, L^1)$ using Lemma \ref{lem:PackingL1}. From the definition of pseudo-dimension, we have $\dim_p(\gC PF_r(G)) \le \dim_p(PF_r(G))$, since any P-shattering set for $\gC PF_r(G)$ will certainly P-shatter $PF_r(G)$. Since $PF_r(G) \subset \gH^n$, we have $\dim_p(PF_r(G)) \le \dim_p(\gH^n) \le n$. Thus $\dim_p(\mathcal{S}) = \dim_p(\gC PF_r(G)) \le n$. $\gS$ is $L^1$-separated with distance at least $\alpha$, and moreover consists of elements that are bounded by $\beta \coloneqq \sup_i \beta_i$. Lemma \ref{lem:PackingL1} now gives:
\begin{equation}\label{eq:EntropyUpperBd2}
    M_\alpha(\mathcal{S}, L^1(\vol_M)) \le e(n+1) \left(\frac{4e\beta\vol(M)}{\alpha}\right)^n.
\end{equation}
Intuitively, $N_r \sim r^{-d}$, so the lower bound \labelcref{eq:EntropyLowerBd2} is exponential in $r$. Meanwhile, $\alpha$ and $\beta$ are both polynomial in $r$, so the upper bound \labelcref{eq:EntropyUpperBd2} is polynomial in $r$. So for sufficiently small $r$, we have a contradiction with the supposition that $\delta \le C_1(r)/4$. We now show this formally. Recall:
\begin{equation}
    \beta = \sup_{i \in [N_r]} \frac{\textcolor{red}{r^k}}{\textcolor{red}{C} \vol_M(B_r(p_i))^{1/p} N_r^{1/p}},\quad \alpha = \frac{\textcolor{red}{r^k} N_r^{1-1/p}\int_0^{r/2}s^{d-1}}{\textcolor{red}{4C} \int_0^{r}s^{d-1} } \inf_{i \in [N_r]}\left[\vol_M(B_{r}(p_i))^{1-1/p}\right].
\end{equation}

Note that the supremum in $\beta$ and the infimum in $\alpha$ is attained by the same $i \in [N_r]$, namely, the $p_i$ that has smallest $\vol_M(B_r(p_i))$. Combining \labelcref{eq:EntropyLowerBd2} and \labelcref{eq:EntropyUpperBd2}, where $s(u) = \sinh(u \sqrt{|K|})$,
\begin{align}
    2^{N_r/16} &\le e(n+1) \left(\frac{4 e\beta \vol(M)}{\alpha}\right)^n \notag\\
    & = e(n+1) \left(\frac{4e\vol(M) \sup_{i \in [N_r]} [\textcolor{red}{r^k}/(\textcolor{red}{C} \vol_M(B_r(p_i))^{1/p} N_r^{1/p})] }{\frac{\textcolor{red}{r^k} N_r^{1-1/p}\int_0^{r/2}s^{d-1}}{\textcolor{red}{4C} \int_0^{r}s^{d-1} } \inf_{i \in [N_r]}\left[\vol_M(B_{r}(p_i))^{1-1/p}\right]}\right)^n \notag\\
    &= e(n+1) \left(16 e \frac{\vol(M)\int_0^{r} s^{d-1} }{N_r\int_0^{r/2} s^{d-1}}\sup_{i,j \in [N_r]} \frac{\vol_M(B_{r}(p_j))^{1/p-1}}{\vol_M(B_r(p_i))^{1/p}}\right)^n \notag \\
    &= e(n+1) \left(16 e \frac{\vol(M)\int_0^{r} s^{d-1} }{N_r\int_0^{r/2} s^{d-1}}\sup_{i \in [N_r]} \left[\vol_M(B_{r}(p_i))^{-1}\right] \right)^n \label{eq:EntropyCombinedBd2}
\end{align}
Note we have the same expression in \labelcref{eq:EntropyCombinedBd2} as \labelcref{eq:EntropyCombinedBd}in the proof for $W^{1,p}(1)$. As before, we get

\begin{align}
    2^{N_r/16} \le e(n+1) \left(2^{2d+4} e C_3 C_2^{-1} \right)^n =  e(n+1) C_4(d)^n.
\end{align}
We get a contradiction if 
\begin{equation}\label{eq:NrContradictionStep2}
    N_r > 16 \left[n \log_2 C_4 + \log_2\left(e(n+1)\right)\right].
\end{equation}
Recalling the lower bound \labelcref{eq:packingBounds} on $N_r$ and using \labelcref{eq:MdKUpperBd},
\begin{align}\label{eq:NrExplicitLowerBound2}
    N_r &\ge \frac{\vol(M)}{\vol_{M^d_K}(B_{2r})} > \frac{\vol(M)}{C_3 (2r)^d}.
\end{align}
Take the following choice of $r$:
\begin{equation}\label{eq:ChoiceR2}
    r = \min\left\{\frac{1}{2}\left(16\frac{C_3}{\vol(M)} \left[n \log_2 C_4 + \log_2(e(n+1))\right]\right)^{-1/d},\, \frac{1}{\sqrt{|K|}},\, \frac{\inj(M)}{2},\, 4 \right\}.
\end{equation}
Using \labelcref{eq:NrExplicitLowerBound2}, this choice of $r$ satisfies the contradiction condition \labelcref{eq:NrContradictionStep2}. Note $r \sim (n + \log n)^{-1/d}$. The constants $C_3, C_4$ depend only on $d$.

\textbf{Step 4. Concluding contradiction.} This choice of $r$ contradicts the assumption that $\delta \le C_1(r)/4$. Therefore, we must have that $\delta > C_1/4$. Since the choice of $r$ is independent of the choice of $\varepsilon>0$ taken at the start of Step 3a, we have that 
\begin{equation}\label{eq:FrGLowerBdC12}
    \dist(F_r(G), \gH^n, L^1(\vol_M)) \ge C_1(r)/4,
\end{equation}
where $r$ is chosen as in \labelcref{eq:ChoiceR2}. We obtain the chain of inequalities
\begin{align}
    \dist(W^{1,p}(1), \gH^n, L^q) &\ge \dist(W^{1,p}(1), \gH^n, L^1) \vol(M)^{1/q-1} \notag\\
    &\ge \dist(F_r(G), \gH^n, L^1)\vol(M)^{1/q-1} \notag\\
    &\ge C_1(r) \vol(M)^{1/q-1}/4 \notag\\
    &= \frac{\textcolor{red}{r^k} N_r^{1-1/p}\int_0^{r/2}s^{d-1}}{\textcolor{red}{8C} \int_0^{r}s^{d-1} } \inf_{i \in [N_r]}\left[\vol_M(B_{r}(p_i))^{1-1/p}\right] \vol(M)^{1/q-1}, 
\end{align}
where the first inequality comes from H\"older's inequality $\|u\|_1 \le \|u\|_q \vol(M)^{1-1/q}$, the second inequality from $F_r(G) \subset W^{1,p}(1)$, the third from \labelcref{eq:FrGLowerBdC12} and the equality from definition \labelcref{eq:C1Def} of $C_1(r)$. We conclude with recalling the bounds \labelcref{eq:MdKUpperBd}, \labelcref{eq:NrExplicitLowerBound}, and \Cref{prop:smallBall}. We have
\begin{align*}
    \dist(W^{1,p}(1), \gH^n, L^q) &\ge \textcolor{red}{\frac{r^k}{8C}} \underbrace{\left(\frac{\vol(M)}{C_3 (2r)^d}\right)^{1-1/p}}_{\text{\labelcref{eq:NrExplicitLowerBound}}} \underbrace{2^{-d}}_{\text{\labelcref{eq:MdKUpperBd}}} \underbrace{\left[C_2 r^d\right]^{1-1/p}}_{\text{Prop. \ref{prop:smallBall}}} \vol(M)^{1/q-1} \\
    &= C_5(d, \vol(M),p,q\textcolor{red}{,g}) r^k.
\end{align*}

The constant is 
\begin{equation}
    C_5 = 2^{-d-3} \frac{\vol(M)^{1/q-1/p}}{2^{d-d/p}\textcolor{red}{C}} (C_2 C_3^{-1})^{1-1/p}.
\end{equation}
Moreover, the constant $C_5$ and choice of $r$ are independent of $\gH^n$. Taking infimum over all choice of $\gH^n$ with $\dim_p(\gH^n) \le n$ and using \labelcref{eq:ChoiceR2}, we have
\begin{equation}
    \rho_n(W^{1,p}(1), L^q) \ge C_5(d, \vol(M),p,q\textcolor{red}{,g}) \textcolor{red}{r^k} \gtrsim (n + \log n)^{-\textcolor{red}{k}/d}.
\end{equation}%

\end{proof}
\end{document}